\documentclass{article} 

\newif\ifarxiv
\arxivfalse
\arxivtrue
\ifdefined\forcearxiv
  \arxivtrue
\fi
\ifdefined\forcecolm
  \arxivfalse
\fi

\PassOptionsToPackage{dvipsnames}{xcolor}
\ifarxiv
  \usepackage[preprint]{colm2026_conference}
\else
  \usepackage[final]{colm2026_conference}
\fi

\usepackage{microtype}
\usepackage{hyperref}
\usepackage{url}
\usepackage{booktabs}
\usepackage{graphicx}
\usepackage{subcaption}
\usepackage{amsmath}
\usepackage{amssymb}
\usepackage{cleveref}
\usepackage{wrapfig}
\usepackage{amsthm}
\usepackage{enumitem}
\usepackage{multirow}

\usepackage{algorithmic}
\usepackage{algorithm}
\usepackage{xcolor}

\newtheorem{theorem}{Theorem}[section]

\newtheorem{lemma}[theorem]{Lemma}

\usepackage{lineno}

\newcommand{\arxiv}[1]{\ifarxiv #1\fi}
\newcommand{\submission}[1]{\ifarxiv\else #1\fi}

\newif\ifshowcomments
\showcommentsfalse 

\definecolor{mayablue}{rgb}{0.21,0.49,0.74}
\NewDocumentCommand{\hl}{ mO{} }{%
  \ifshowcomments
    \textcolor{mayablue}{\textsuperscript{\textit{Hanlin}}\textsf{\textbf{\small[#1]}}}%
  \fi
}

\NewDocumentCommand{\yc}{ mO{} }{%
  \ifshowcomments
    \textcolor{RedOrange}{\textsuperscript{\textit{Yiling}}\textsf{\textbf{\small[#1]}}}%
  \fi
}

\NewDocumentCommand{\shi}{ mO{} }{%
  \ifshowcomments
    \textcolor{Periwinkle}{\textsuperscript{\textit{Shi}}\textsf{\textbf{\small[#1]}}}%
  \fi
}

\definecolor{darkblue}{rgb}{0, 0, 0.5}
\hypersetup{colorlinks=true, citecolor=darkblue, linkcolor=darkblue, urlcolor=darkblue}

\title{Peer-Predictive Self-Training for Language Model Reasoning}

\author{Shi Feng\thanks{Equal contribution.} \\
Harvard University \\
\texttt{shifeng-thu@outlook.com}
\And
Hanlin Zhang\footnotemark[1] \\
Harvard University \\
\texttt{hanlinzhang@g.harvard.edu}
\And
Fan Nie \\
Stanford University \\
\texttt{niefan@stanford.edu}
\And
Sham Kakade \\
Harvard University \\
\texttt{sham@seas.harvard.edu}
\And
Yiling Chen \\
Harvard University \\
\texttt{yiling@seas.harvard.edu}
}

\begin{document}

\ifcolmsubmission
\linenumbers
\fi

\maketitle
\ifarxiv
\lhead{Preprint}
\fi

\newcommand{\method}{{PST}}

\begin{abstract}

Mechanisms for continued self-improvement of language models without external supervision remain an open challenge. 
We propose {\em Peer-Predictive Self-Training} (\method), a label-free fine-tuning framework in which multiple language models improve collaboratively by leveraging a cross-model aggregated response as an internal training signal. Given a prompt question, the models generate responses sequentially; the final aggregated answer--often more reliable than individual responses in practice--serves as an internal reference for learning. We measure how informative each intermediate response is about the aggregate using pointwise mutual information (PMI), and use this signal to scale self-training updates. Responses already aligned with the aggregate are updated less, while less informative or misaligned responses are updated more. On mathematical reasoning benchmarks (SimulEq, MATH-500-Numeric, and MultiArith),
\method\ improves exact-match accuracy by 2.2--4.3 percentage points across Gemma-2-2B, LLaMA-3.2-1B, and Qwen2.5-1.5B, and reduces the average
generator--verifier gap (GV-Gap) \citep{song2024mind} by 26--40\%, while requiring no external supervision or teacher--student hierarchy and relying solely on cross-model interactions. These results suggest that cross-model generations and peer-predictive feedback can serve as an effective approach for self-supervised training.
\end{abstract}

\section{Introduction}

Continued self-improvement of large language models (LLMs) without external supervision is a central open problem.
Many successful post-training pipelines rely on labeled data, explicit reward models, or carefully engineered feedback signals
\citep{qievolm,wang2025benefits,lu2025verification}.
While effective in controlled settings, these methods fundamentally limit scalability and adaptability, especially when
ground truth is expensive, incomplete, or unavailable, and when improvement must proceed continually without assuming
access to stronger external supervision \citep{saunders2022self, bowman2022measuring}.

A growing body of recent work therefore explores self-improvement driven by internally generated signals, including
self-training and multi-model interaction \citep{song2024mind,sun2025theoretical}.
These approaches have proven particularly effective for \emph{reasoning-oriented tasks}, where correct solutions often
require multi-step logical consistency rather than surface-level pattern matching \citep{wang2024alpine,wang2025benefits}.
In these settings, models may struggle to reliably generate correct answers, yet still retain the ability to recognize or
validate correct reasoning trajectories when presented with candidate solutions, reflecting a fundamental asymmetry
between generation and verification \citep{song2024mind,sun2025theoretical}.
This verification--generation asymmetry makes reasoning problems especially amenable to self-training based on internal
feedback signals.

Recent work further shows that homogeneous model populations can exhibit emergent coordination and collective behavior
through open-ended interaction, even in the absence of external labels \citep{jiang2025artificial}.
However, how to elicit reliable and stable supervision signals in a fully unsupervised manner remains unresolved,
particularly in the continual self-improvement regime where error accumulation and confirmation bias can severely limit
long-term gains \citep{song2024mind}.

\begin{wrapfigure}{r}{0.55\linewidth}
\centering
\includegraphics[width=\linewidth]{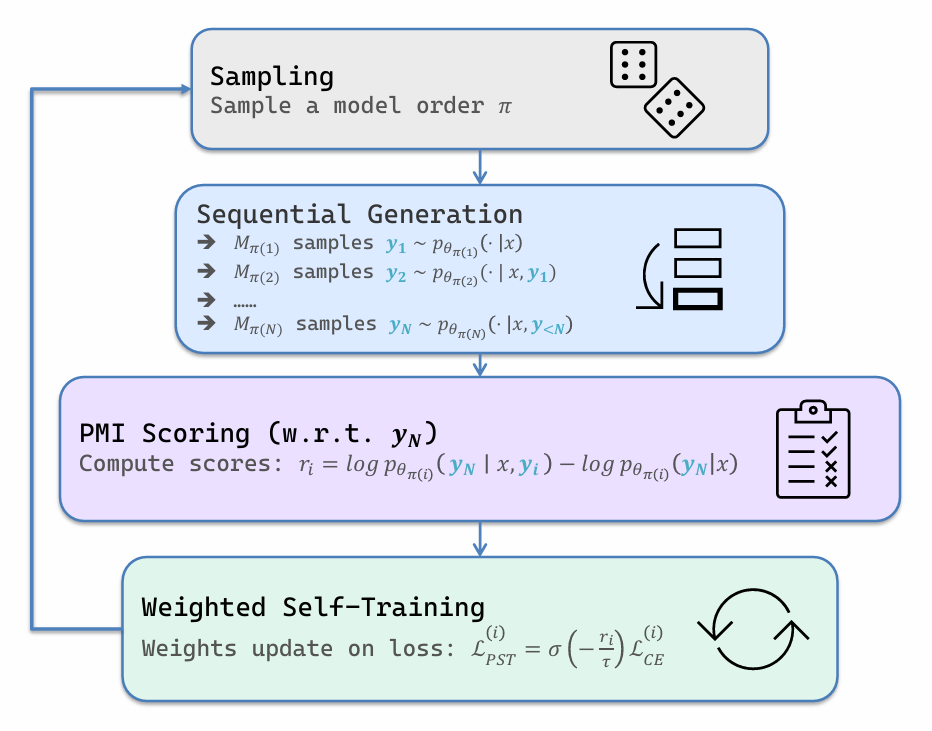}
\caption{\textbf{Overview of the PST framework.}
\textbf{Grey:} sample models and randomize their order.
\textbf{Blue:} generate responses sequentially conditioned on prior outputs, yielding $y_N$.
\textbf{Purple:} score each response by its contribution to predicting $y_N$.
\textbf{Green:} scale self-training updates using weights from these scores, correcting misaligned responses more.}
\label{fig:flow}
\end{wrapfigure}

A key insight from both learning theory and economics is that \emph{reliability can emerge from aggregation},
a phenomenon often referred to as the \emph{wisdom of crowds}.
Across domains, aggregating multiple independent or weakly correlated predictions can effectively pool dispersed
information and yield more accurate outcomes than relying on a single predictor
\citep{song2024mind,lu2025verification,Wol:04,Sur:04}.
In the context of language models, different models--or different reasoning trajectories produced by the same
model--may encode complementary partial information about the underlying solution, such that aggregation implicitly
filters noise while amplifying consistent and informative signals.
As a result, the aggregated output is frequently more reliable than most individual responses.
Recent theoretical work formalizes this intuition through the \emph{coverage principle}, which argues that post-training
or selection-based improvement is possible as long as the base model assigns non-negligible probability mass to correct
solutions, even if they are rare \citep{chen2025coverage}.
Under this view, the central challenge is not the absence of correct solutions, but how to identify and amplify them without access to ground truth \citep{zhou2020towards, zelikman2022star}.

This challenge is closely related to prior work on information elicitation under unverifiable outcomes, which studies
how informative signals can be extracted from cross-agent statistical structure without relying on ground truth \citep{miller2005eliciting, shnayder2016informed, kong2019information, kong2020dominantly, feng2022peer, chen2024carrot, chen2025datareliability}.
Rather than external labels, such approaches leverage mutual predictiveness across agents, and recent work shows that
mutual information--based signals can reliably evaluate LLM judgments in no-gold-standard settings
\citep{xubenchmarking}.
Taken together, these observations suggest a natural synthesis of aggregation and peer-based evaluation: a pool of
heterogeneous language models can not only aggregate dispersed beliefs into a stronger collective answer, but also
provide informative cross-model signals about answer quality even when no external supervision is available.

Building on this perspective, we propose \emph{Peer-Predictive Self-Training} (PST), a cooperative self-improvement
framework for language model reasoning.
Given a prompt, multiple language models generate responses sequentially, and the final aggregated response serves as a
crowd-based internal reference signal.
Rather than treating this aggregate as a hard pseudo-label, PST quantifies how informative each intermediate response is
about the final aggregate using pointwise mutual information (PMI).
This peer-predictive signal is then used to adaptively modulate the strength of self-training updates: responses that already align
with the collective decision receive smaller updates, while less informative or misaligned responses are updated more
strongly.

The key intuition behind PST is that aggregation induces a form of lightweight verification signal, as previously studied in in-context learning \citep{zhang2024study}.
By pooling diverse reasoning trajectories, the aggregated response often reflects a stronger evaluation signal than any
single generator alone, even in the absence of an explicit verifier.
Crucially, PST is not an aggregation method, but a mechanism for converting peer-predictive information into supervision signals.
PST leverages this verification--generation asymmetry to guide self-improvement, encouraging models to internalize
reasoning patterns aligned with the strongest implicit peer feedback.
Importantly, PST requires no ground-truth supervision, no reward models, and no fixed teacher--student hierarchy,
relying solely on cross-model interaction and peer predictiveness.

We evaluate PST on mathematical reasoning benchmarks including MATH-500-Numeric, MultiArith, and SimulEq, using three
heterogeneous instruction-tuned model families: Gemma-2-2B, LLaMA-3.2-1B, and Qwen2.5-1.5B. We intentionally focus on mathematical reasoning, as it provides a clean and objectively verifiable setting for studying self-improvement, allowing us to isolate the effect of the learning mechanism without confounding factors such as subjective evaluation.
Across all settings, PST yields consistent improvements of 2.2--4.3 percentage points in exact-match accuracy over the
original models and reduces the average generator--verifier gap by 26--40\%.
While the absolute gains are modest, they are consistent across models of different scales and mathematical reasoning datasets, and are achieved without any
external supervision or verifier training.
These results demonstrate that peer-predictive feedback provides a simple yet effective foundation for continued,
fully unsupervised self-improvement of language models on reasoning tasks. Additional experiments on larger-scale models, including Gemma 3 (4B), LLaMA 3.1 (8B), and Qwen2.5 (3B), are provided in Appendix~\ref{app:larger_scale}, where we observe similar improvements, further validating the effectiveness of \method\ at larger scales. We further provide additional baseline comparisons and epoch-wise ablations in Appendix~\ref{app:more_baselines} and Appendix~\ref{app:epoch_ablation}, showing that the gains are not explained by self-consistency, ordinary self-training, or majority-vote pseudo-labeling, and that improvement tends to saturate over repeated PST epochs.


\section{Related Work}

Recent work has examined the limits of self-improvement in language models through the lens of generator--verifier interactions. 
\citet{lu2025verification} show that generator--verifier similarity fundamentally constrains gains from self-verification, and that distributional diversity among models is critical for effective verification. 
This insight directly motivates the use of heterogeneous peers in our framework.

Beyond verifier-based selection, a growing body of work studies self-improvement without external supervision. 
Classical formulations of self-training and semi-supervised learning date back to bootstrapping and co-training. 
Co-training \citep{blum1998combining, nigam2000analyzing} exploits multiple views, while self-training and related variants \citep{yarowsky1995unsupervised, zhou2005tri} iteratively refine predictions from unlabeled data. 
Modern approaches such as pseudo-labeling and consistency-based methods \citep{lee2013pseudo, sohn2020fixmatch} scale these ideas to deep models. 
However, these methods rely on fixed views, confidence thresholds, or augmentation assumptions, and do not explicitly leverage adaptive cross-model interactions.
Classical self-training and bootstrapping methods on language models iteratively fine-tune models on their own predictions, but are prone to confirmation bias and error reinforcement when incorrect generations are amplified \citep{song2024mind, kadlvcik2024self}. 
Recent approaches mitigate this issue by relying on coverage assumptions, arguing that post-training is effective as long as the base model assigns non-negligible probability mass to correct solutions, even if they are rare \citep{chen2025coverage}. 
However, these methods typically operate in a single-model regime and do not explicitly exploit cross-model interactions. 
Recent work has also explored learning from multiple models or peer interactions. 
For example, \citet{wu2026reasoning} and \citet{luo2025learning} study how models can improve by leveraging signals from other models or agents, often through interaction or selection at inference time. 
These approaches highlight the potential of cross-model learning. 
In contrast, PST focuses on using peer agreement to derive a quantitative signal that can guide parameter updates during training.

Another related line of work explores collective behavior and multi-model interaction.
\citet{jin2025discovering, jiang2025artificial} study homogeneous populations of language models and show that open-ended interaction can induce emergent coordination, while also revealing limitations due to model similarity.
Debate-style or discussion-based methods, as well as aggregation-based methods such as majority voting and self-consistency \citep{xie2020self}, rely on explicit comparison rules or external judges to improve performance at inference time, but do not provide a training signal that can be used to update model parameters.
In contrast, Peer-Predictive Self-Training (PST) derives supervision implicitly from aggregation and peer predictiveness, without explicit debate, ranking, or teacher--student roles.

Our method is closely related to peer prediction and information elicitation mechanisms, which aim to recover truthful signals when ground truth is unavailable \citep{miller2005eliciting, kong2020dominantly}.
Recent work demonstrates that mutual information--based signals can reliably evaluate LLM judgments without gold labels \citep{xubenchmarking}.
PST operationalizes this idea in a training loop: pointwise mutual information quantifies how much a model's response contributes to a stronger collective decision, and this signal directly modulates the strength of self-training updates.

From a generator--verifier perspective, PST builds on analyses of generator--verifier gaps in LLM self-improvement \citep{song2024mind,sun2025theoretical,saad2025shrinking}.
Prior work emphasizes that improvement requires access to verifiers strictly stronger than the generator.
Recent work shows that even relatively weak verifiers can help reduce the generation--verification gap \citep{saad2025shrinking}.
Unlike approaches that assume an explicit verifier, PST induces a stronger evaluation signal through aggregation across peers.
Empirically, this leads to reduced generator--verifier gaps after training, and theoretically, to monotone convergence driven by cross-model capability asymmetries.

Finally, our theoretical motivation aligns with recent results on the computational asymmetry between generation and verification.
Verification can often be implemented by shallow or low-precision transformers, whereas long-horizon planning and generation are computationally harder \citep{merrill2023parallelism,wang2024alpine,wang2025benefits}.
PST leverages this asymmetry implicitly: even when individual models struggle to generate correct solutions, their combined verification signals can still provide a reliable self-supervised training signal.

\section{Peer-Predictive Self-Training (\method)}
\label{sec.model}

We first introduce our Peer-Predictive Self-Training (\method) framework for multi-model self-improvement. \method\ is guided by two ideas: (i) {\em wisdom of crowds} -- a pool of heterogeneous language models can aggregate dispersed ``beliefs'' and produce an aggregated answer that is usually stronger than individual responses, akin to how prediction markets aggregate information; (ii) {\em peer prediction} -- when no ground truth is available, cross-agent consistency still provides valuable information for answer quality. Concretely, \method\ scores each non-final model's response by how informative it is about the final aggregate via pointwise mutual information (PMI) and uses this peer-predictive signal to weight self-training updates. Figure~\ref{fig:flow} provides an overview of the PST pipeline.

{\bf Cross-model generation.} Let $\{M_k\}_{k=1}^N$ be a set of $N \ge 2$ autoregressive language models with parameters $\{\theta_k\}_{k=1}^N$. At the start of each epoch, we sample a random permutation $\pi$ of $\{1, \dots, N\}$. The model acting at position $i$ in the resulting order is $M_{\pi(i)}$ with parameters $\theta_{\pi(i)}$.

Given a question prompt $x$, the models produce responses sequentially.  
The $i$-th model $M_{\pi(i)}$ produces a response $y_i$ conditioned on $x$ and all 
previous responses $y_{<i} = (y_1,\ldots,y_{i-1})$:
\(
  y_i \sim p_{\theta_{\pi(i)}}(\cdot \mid x, y_{<i}). 
\)
We treat the final response $y_N$ as an ensemble-aggregated answer, since it is produced after observing all previous candidate solutions. 

{\bf Peer-predictive score via PMI.}
Taking $y_N$ as an internal reference signal, we want to quantify how much each intermediate response $y_i$ helps ``predict'' the final aggregate $y_N$.
For each non-final position $i\in\{1,\dots,N-1\}$, we compute a pointwise mutual information (PMI)~\cite{xubenchmarking} score under the same model $M_{\pi(i)}$ that produced $y_i$:
\[
  r_i= \log p_{\theta_{\pi(i)}}(y_N \mid x, y_i)- \log p_{\theta_{\pi(i)}}(y_N \mid x).
\]
Intuitively, $r_i$ is large when model $M_{\pi(i)}$'s own response $y_i$ strongly increases its likelihood of producing the final aggregate $y_N$, relative to its baseline from the prompt alone. 

\begin{algorithm}[htbp]
\caption{Peer-Predictive Self-Training}
\label{alg:pft_brief}
\begin{algorithmic}
\STATE \textbf{Input:} dataset $\mathcal{D}$; models $\{M_k\}$ with params $\{\theta_k\}$; temperature $\tau$; epochs $E$
\FOR{$e=1$ to $E$}
  \STATE Randomly generate a permutation $\pi$ of $\{1, 2, \dots N\}$
  \FOR{each prompt $x \in \mathcal{D}$}
    \STATE $y_{<1} \leftarrow \varnothing$
    \FOR{$i=1$ to $N$}
      \STATE sample $y_i \sim p_{\theta_{\pi(i)}}(\cdot \mid x, y_{<i})$
      \STATE $y_{<i+1} \leftarrow (y_{<i}, y_i)$
    \ENDFOR
    \STATE Let $y_N$ be the final aggregated response
    \FOR{$i=1$ to $N-1$}
      \STATE $r_i \leftarrow \log p_{\theta_{\pi(i)}}(y_N \mid x, y_i) - \log p_{\theta_{\pi(i)}}(y_N \mid x)$, 
      $\alpha_i \leftarrow \sigma(-r_i / \tau)$
      \STATE $\mathcal{L}^{(i)}_{\mathrm{CE}} \leftarrow -\sum_t \log p_{\theta_{\pi(i)}}(y_{i,t+1}\mid x, y_{<i}, y_{i,\le t})$, 
      $\mathcal{L}^{(i)}_{\mathrm{PST}} \leftarrow \alpha_i \cdot \mathcal{L}^{(i)}_{\mathrm{CE}}$
      \STATE Update parameters $\theta_{\pi(i)}$ by backpropagating $\mathcal{L}^{(i)}_{\mathrm{PST}}$ with AdamW 
    \ENDFOR
    \STATE Update $\theta_{\pi(N)}$ by standard next-token prediction on $y_N$
  \ENDFOR
\ENDFOR
\STATE \textbf{Output:} fine-tuned parameters $\{\theta_k\}$
\end{algorithmic}
\end{algorithm}


We convert $r_i$ into a non-negative update weight using a temperature-controlled sigmoid:
\[
  \alpha_i
  = \sigma\!\left( -\frac{r_i}{\tau} \right),
  \qquad
  \sigma(u)
  = \frac{1}{1 + e^{-u}},
\]
where temperature $\tau>0$ controls sensitivity.  
Higher PMI therefore yields smaller $\alpha_i$, resulting in a smaller update.  
The idea is that a model whose response already aligns with the aggregated 
solution should change less, which encourages a natural division of labor 
instead of redundant reasoning.

{\bf Weighted self-training.}
Each non-final model $M_{\pi(i)}$ is trained on its own generated trajectory $y_i = (y_{i,1},\ldots,y_{i,T_i})$ with next-token cross-entropy loss:
\[
  \mathcal{L}_{\mathrm{CE}}^{(i)}
  = -\sum_{t=1}^{T_i-1}
      \log p_{\theta_{\pi(i)}}\!\big(
         y_{i,t+1}
         \mid x, y_{<i}, y_{i,\le t}
      \big).
\]
We scale this loss by the peer-predictive weight:
\(
  \mathcal{L}_{\mathrm{PST}}^{(i)}= \alpha_i\, \mathcal{L}_{\mathrm{CE}}^{(i)}. 
\)
We then backpropagate through $\mathcal{L}_{\mathrm{PST}}^{(i)}$ and update parameters $\theta_{\pi(i)}$ (e.g., with AdamW). The final model is updated by standard next-token prediction on $y_N$. The permutation $\pi$ is resampled each epoch so that different models take on different positions--including the final ``aggregator'' position--over training. Algorithm \ref{alg:pft_brief} summarizes the \method\ framework. The final response $y_N$ can be viewed as an aggregated answer, as the last model conditions on all previous responses and thus integrates information across models without requiring an external aggregation module. Randomizing the permutation ensures that each model is equally likely to appear at any position in expectation, mitigating position bias.


\section{Experiments}
\label{sec.exp}

\begin{figure*}[t]
  \centering
    \begin{subfigure}[t]{0.32\linewidth}
    \centering
    \includegraphics[width=\linewidth]{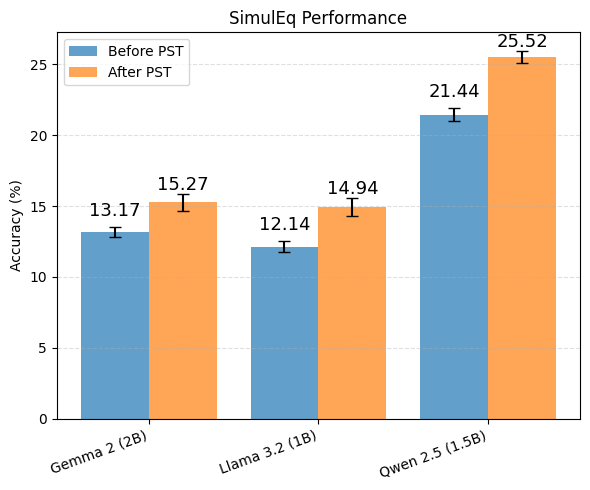}
    \caption{SimulEq}
    \label{fig:accuracy_simuleq}
  \end{subfigure}
   \hfill
  \begin{subfigure}[t]{0.32\linewidth}
    \centering
    \includegraphics[width=\linewidth]{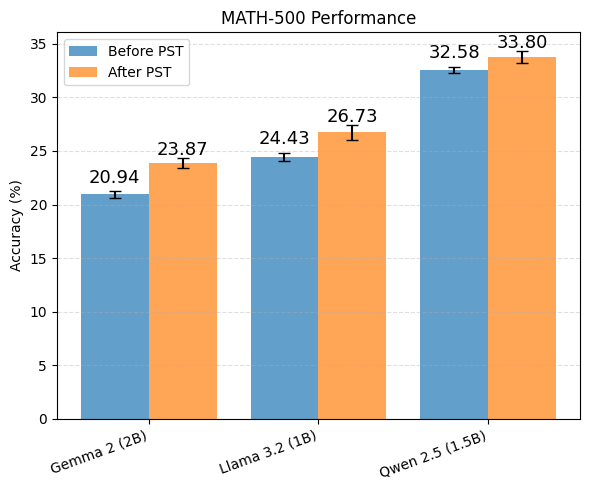}
    \caption{MATH-500-Numeric}
    \label{fig:accuracy_math500}
  \end{subfigure}
  \hfill
  \begin{subfigure}[t]{0.32\linewidth}
    \centering
    \includegraphics[width=\linewidth]{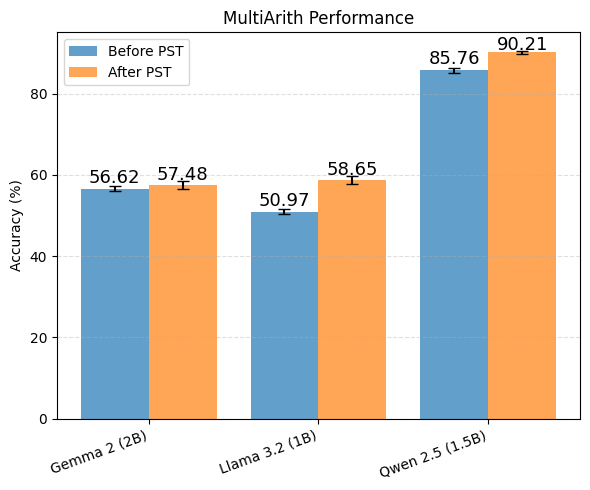}
    \caption{MultiArith}
    \label{fig:accuracy_multiarith}
  \end{subfigure}

  \caption{
  Accuracy before and after fine-tuning on SimulEq, MATH-500-Numeric, and MultiArith.
  Each panel reports pre- and post-fine-tuning performance under \method\ for
  Gemma-2-2B,  LLaMA-3.2-1B, and Qwen2.5-1.5B.
  }
  \label{fig:accuracy_all}
\end{figure*}

\paragraph{Experimental Setup.}
We evaluate PST using three instruction-tuned LMs with different architectures and sizes: 
\textbf{Gemma-2-2B-Instruct}~\cite{gemma2_2b_it},
\textbf{LLaMA-3.2-1B-Instruct}~\cite{llama32_1b_instruct},
and \textbf{Qwen2.5-1.5B-Instruct}~\cite{qwen25_15b_instruct}.

We fine-tune each base model using LoRA adapters~\cite{hu2022lora},
while keeping all original model parameters frozen.
Training updates are applied exclusively to the LoRA parameters.

Model optimization is carried out using AdamW~\cite{loshchilov2019adamw}
with a learning rate of $1\times10^{-6}$ and weight decay $0.01$.
We adopt a cosine-annealing learning-rate schedule, with the maximum
annealing horizon set to the total number of training epochs.
Forward computation is performed under bfloat16 automatic mixed
precision.

For each training instance, models generate responses sequentially.
PMI scores $r_i$ are computed for all non-final models, and the training
objective for each non-final model $M_{\pi(i)}$ is defined as a scaled next-token
cross-entropy loss $\alpha_i \mathcal{L}_{\mathrm{CE}}$, where
$\alpha_i = \sigma(-r_i / \tau)$ and $\tau = 3.0$.
The final model is updated by standard next-token prediction on the
aggregated response it produces.
Gradients are clipped to a global norm of $1.0$, and optimization steps
are skipped whenever non-finite gradients are detected.
At the beginning of each epoch, the ordering of models is reshuffled so
that different models assume the final (aggregator) role across epochs.
Each model maintains its own optimizer and scheduler, and training is
conducted for a total of $5$ epochs.

\paragraph{Datasets.}
We consider 3 standard reasoning benchmarks:

\begin{itemize}[leftmargin=*]
\item \textbf{SimulEq}~\cite{kushman2014learning}, 
    a collection of problems focused on solving systems of simultaneous equations of variable complexity.
    \item \textbf{MATH-500-Numeric}~\cite{hendrycks2021measuring}, 
    a challenging benchmark of mathematical problem solving derived from the broader MATH dataset
    that contains 500 problems with integer answers.

    \item \textbf{MultiArith}~\cite{roy2016solving}, 
    a dataset of multi-step arithmetic word problems requiring combination of basic operations such as
    addition, subtraction, multiplication, and division.
\end{itemize}

These benchmarks provide a clean and verifiable setting for isolating the effect of the learning mechanism.
All datasets are formatted uniformly as
\texttt{Question: <problem>} and presented to all models sequentially.
We record each model's raw response (given only the question as a prompt)
before and after fine-tuning using \method, and compute the exact-match
accuracy for each individual model, with error bars.

\paragraph{Evaluation and Figures.}
After five epochs of \method\  fine-tuning, we evaluate model performance on  
SimulEq, MATH-500-Numeric, and MultiArith.  
Figure~\ref{fig:accuracy_all} shows paired bars representing  
pre- and post-fine-tuning accuracy for each model.  

In all benchmarks, fine-tuning using \method\ consistently improves exact-match accuracy across all models.
Averaged over models, \method\ yields a mean accuracy improvement of approximately 3.0 percentage points on SimulEq, 2.2 percentage points on MATH-500-Numeric, and 4.3 percentage points on MultiArith.
These gains are obtained using standard fine-tuning settings without external supervision, indicating that PMI-based modulation alone provides a stable and effective training signal. Together, these results indicate that \method\ effectively enhances reasoning
reliability across heterogeneous models by allocating larger updates to  
less informative responses and smaller updates to highly predictive ones.

\section{Why Does \method\  Work?}
\label{sec:why-pst}

\subsection{Intuition}
\label{subsec:intuition}

\method\ can be viewed as self-training with an \emph{endogenously improved} reference signal and an \emph{information-weighted} update rule. For each prompt, multiple models generate candidate solutions sequentially, and the final response (conditioned on all prior candidates) acts as an aggregated reference. \method\ then updates each model on its own trajectory, but scales the update magnitude using a PMI-based signal. Intuitively, if a model's response already makes the aggregate highly likely (high PMI), then it is already aligned with the crowd and should change little; if it is weakly predictive or misaligned (low PMI), the model receives a larger update, pushing it toward behaviors that better match the aggregated decision.


A helpful lens is to view each model as playing both a generator and a
verifier role. During sequential generation, model $M_{\pi(i)}$ produces
\(
y_i \sim p_{\theta_{\pi(i)}}(\cdot \mid x, y_{<i}),
\)
where conditioning on $y_{<i}$ implicitly reflects an evaluation of earlier candidates before producing $y_i$. The final output $y_N$ pools information from all peers and can be interpreted as the result of applying the strongest available implicit verification signal within the group to the set of proposed candidates.

Why can $y_N$ be a stronger reference signal than any single model's output? As discussed in Appendix~\ref{sec.hardness}, verification and generation can differ sharply in computational difficulty: when abstracting a reasoning problem as a path-finding problem on a directed graph, verifying a proposed path can lie in a low-depth circuit class ($\mathrm{TC}^0$), while constructing a correct path can be substantially harder ($\mathrm{P}$- or $\mathrm{NL}$-hard). This verification--generation asymmetry implies that even when individual models struggle to \emph{generate} correct solutions from scratch, they may still collectively provide a comparatively strong \emph{evaluation} signal when multiple candidate solutions are available.



\method\ uses this aggregated output $y_N$ as a self-generated reference and
assigns each response $y_i$ a reward
\(
r_i = \mathrm{PMI}(y_N, y_i),
\)
which measures how informative $y_i$ is about the final decision.
Correct or on-track responses tend to have high PMI and therefore receive smaller corrective weights, since they are already predictive of the aggregate. Low-PMI responses are less aligned with the final decision and receive larger self-training weights, causing the corresponding models to adapt more strongly.

From a generator--verifier perspective, this training process encourages each
model to generate candidates that align with the strongest implicit verifier in
the peer group.
Equivalently, \method\ drives models to internalize the behavior of the best
available verification signal.
In \Cref{app:gv}, we quantify this effect using an empirical generator--verifier
(GV) gap metric and observe that \method\ consistently reduces this gap, indicating
improved generator--verifier alignment through peer-predictive self-training.

\subsection{Theoretical Analysis}
\label{subsec:theory}

We analyze \method\ at the level
of aggregate model capabilities, following the generator--verifier viewpoint of prior work
\cite{sun2025theoretical}.
Our objective is not to characterize the underlying optimization procedure, but
to understand how relative capability gaps drive improvement.

We consider $n$ models trained jointly.
For each model $i \in \{1,\dots,n\}$, we associate a scalar capability variable
\(
s_i(t) \in [0,1],
\)
representing its expected performance on the target task at time $t \ge 0$.
These variables are abstract and summarize population-level performance rather
than specific parameters or losses.

\method\ produces an effective self-reward signal by aggregating model outputs.
Rather than modeling individual verifiers explicitly, we represent this signal
by a single scalar
\(
v(t) \in [0,1],
\)
which captures the strength of the aggregated evaluation signal used for
training.
Motivated by the wisdom-of-the-crowd effect and the generator--verifier
asymmetry discussed in Appendix~\ref{sec.hardness}, we allow the aggregated
evaluation signal to be initially stronger than individual generators, i.e.,
\(
v(0) > \max_i s_i(0).
\)

Model improvement is assumed to be gap-driven: a model can only improve when
trained against an evaluation signal that is strictly stronger than itself.
In addition, learning saturates as capabilities approach their upper bound.
Under these assumptions, we model \method\ by the following dynamical system:
\begin{align}
\frac{d s_i}{dt}
=
\alpha (1 - s_i)\,[v - s_i]_+, 
\frac{d v}{dt}=
\beta (1 - v)\,\max_i s_i,
\label{eq:pst_dyn}
\end{align}
where $\alpha, \beta > 0$ are scaling constants and $[x]_+ := \max\{x,0\}$.
The first equation captures improvement driven by positive generator--verifier
gaps, while the second reflects that the aggregate evaluation signal strengthens
as generators improve and saturates near its maximum.

We now characterize the long-run behavior of the system.

\begin{theorem}[Convergence of \method\ dynamics]
\label{thm:pst_convergence}
Assume $s_i(0), v(0) \in [0,1]$ for all $i$.
Let $(s_1(t),\dots,s_n(t),v(t))$ follow~\eqref{eq:pst_dyn}.
Then:
\begin{enumerate}
\item For $\forall i$, $s_i(t)$ is monotone non-decreasing and
$s_i^\ast := \lim_{t \to \infty} s_i(t)$ exists with $s_i^\ast \ge s_i(0)$.
\item $v(t)$ is monotone non-decreasing and the limit
$v^\ast := \lim_{t \to \infty} v(t)$ exists with $v^\ast \ge v(0)$.
\item For any model $i$ with $s_i(0) < v^\ast$, we have $s_i^\ast = v^\ast$.
In particular, at convergence no model remains strictly below the limiting
evaluation signal unless it started at or above it.
\end{enumerate}
\end{theorem}

The dynamics induced by \method\ are governed by monotonicity and boundedness.
Each solver's capability increases whenever it falls below the current
aggregated evaluation signal and is upper bounded, ensuring convergence.
The aggregated signal itself evolves monotonically as a function of the
strongest solver.
Intuitively, any solver that were to remain strictly below the limiting
aggregation level would continue to receive a positive update signal,
making such a gap unsustainable at convergence.
As a result, \method\ induces monotone capability improvement driven by gaps
relative to a strong aggregated evaluation signal, and converges to a
saturated state in which self-reward no longer provides additional learning
signal.
A formal proof is provided in Appendix~\ref{app.proof_pst_convergence}.

Empirically, we observe qualitative behavior consistent with this gap-driven perspective.
Across all datasets, the row-wise maximum GV gaps--which capture the strongest verifier available to each generator--consistently decrease after training.
This indicates that generators progressively align with stronger peer verifiers, reducing the gap between their own capability and the strongest available evaluation signal.
In contrast, single-model self-training is limited to the generator's own single-sample behavior and does not exploit cross-verifier selection signals.
Detailed definitions, metrics, and full empirical results are provided in \Cref{app:gv}, which further show that the observed improvements go beyond what standard self-training can achieve.
These results provide empirical support for the monotone improvement dynamics and convergence behavior predicted by our analysis.

\arxiv{\subsection{Empirical Improvements on GV-Gap}
\label{app:gv}

\begin{figure}[htbp]
    \centering
    \begin{subfigure}[b]{0.32\linewidth}
        \includegraphics[width=\linewidth]{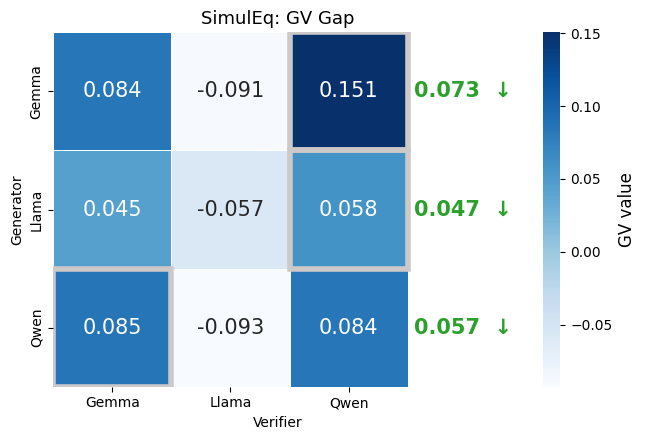}
        \caption{SimulEq}
    \end{subfigure}
    \hfill
    \begin{subfigure}[b]{0.32\linewidth}
        \includegraphics[width=\linewidth]{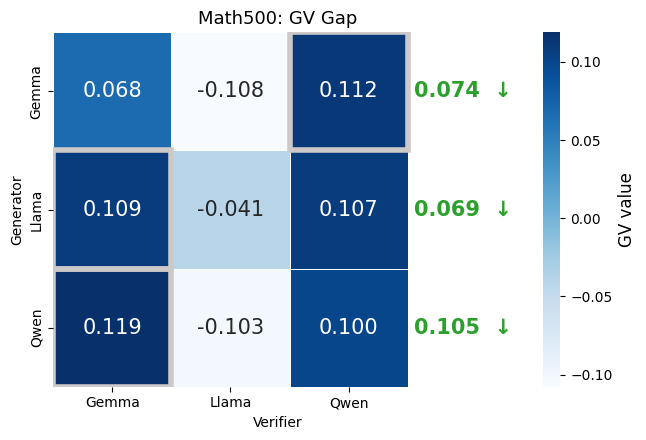}
        \caption{MATH-500-Numeric}
    \end{subfigure}
    \hfill
    \begin{subfigure}[b]{0.32\linewidth}
        \includegraphics[width=\linewidth]{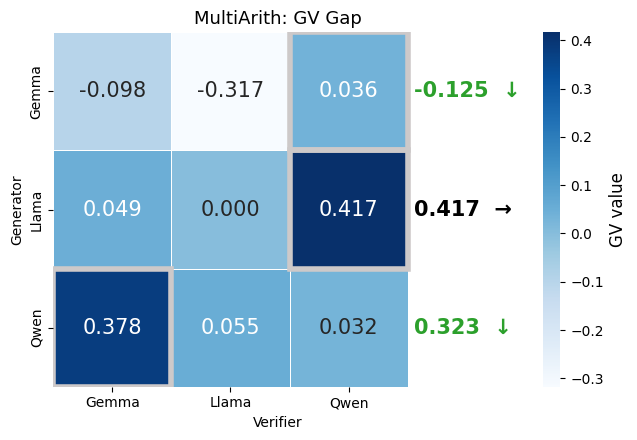}
        \caption{MultiArith}
    \end{subfigure}

    \caption{
    Initial GV matrices on SimulEq, MATH-500-Numeric, and MultiArith, computed using
    initial models as both generators and verifiers.
    Row-wise maxima at initialization are highlighted by gray boxes.
    The \textcolor{ForestGreen}{green} numbers on the right report the
    corresponding row-wise maximum GV gaps after fine-tuning, computed
    between fine-tuned generators and fine-tuned verifiers
    (\textcolor{ForestGreen}{green} indicates a decrease, while black indicates no change).
    }
    \label{fig:gv_initial}
\end{figure}

The analysis above predicts that the gains of \method\ arise from
cross-model interactions, in which stronger verifiers provide a
cleaner training signal.

Empirically, for a generator $M_g$ and a verifier $M_v$, given a
question--answer pair $(x,a)$, we sample $k$ candidate responses
$\{y^{(1)},\dots,y^{(k)}\}\sim p_{M_g}(\cdot\mid x)$ from the generator.
Let $y^{(1)}$ denote the generator's ordinary single-sample output. The
verifier selects the highest-scoring candidate from the same candidate set,
and we define the generator--verifier gap on $x$ as
\begin{align*}
y_v^\star
&= \arg\max_{t\in[k]} s_{M_v}(x,y^{(t)}),\\
\mathrm{GV\text{-}Gap}(M_g,M_v\mid x)
&\triangleq
\mathbf{1}\{\mathrm{Ans}(y_v^\star)=a\}
-
\mathbf{1}\{\mathrm{Ans}(y^{(1)})=a\}.
\end{align*}
Thus, the GV gap measures the additional correctness obtained by
verifier-based selection over the generator's single-sample output.

Here $s_{M_v}(x,y)\in\mathbb{R}$ denotes the scalar score assigned by the
verifier model $M_v$ to a completed response $y$ for input $x$. The verifier
is run in evaluation mode, performing deterministic forward passes on fixed
input--response pairs without sampling, dropout, or parameter updates. The
score is computed as a sequence-level value by summing token-level
log-probabilities over the response tokens and is used only to rank candidate
responses, following the generator--verifier framework of
\cite{song2024mind}. We report the dataset-level GV gap by averaging
$\mathrm{GV\text{-}Gap}(M_g,M_v\mid x)$ over prompts $x$.

Figure~\ref{fig:gv_initial} shows the initial GV matrices on SimulEq,
MATH-500-Numeric, and MultiArith.
Rows correspond to generators and columns to verifiers.
For each generator, the row-wise maximum GV gap at initialization is
highlighted by a gray box.
In most cases, this maximum occurs off the diagonal and is strictly
positive, indicating that a generator aligns more strongly with another
model acting as a verifier than with itself.
This reveals systematic cross-model asymmetries that \method\ can
exploit during training, suggesting that peer interaction provides
stronger supervision than self-training within a single model. 
This aligns with prior findings in \citep{song2024mind, lu2025verification} that cross-verification from diverse models is effective.

The number shown to the right of each row reports the corresponding
row-wise maximum GV gap after fine-tuning, where fine-tuned models are
used as both generators and verifiers.
Across all datasets, these values consistently decrease after training.
Specifically, the sum of row-wise maximum GV gaps is reduced by approximately 39.8\% on SimulEq, 27.1\% on MATH-500-Numeric, and 26.0\% on MultiArith, indicating that generators internalize strategies aligned with the strongest verifiers.
In contrast, standard self-training relies on each model's own predictions and evaluations, whose generator--verifier value is inherently bounded by its own capability and thus cannot match the maximum GV gap achieved by stronger peer models.
By leveraging cross-model interactions, PST allows each generator to align with signals induced by stronger verifiers, leading to systematically reduced generator--verifier gaps.
This behavior is consistent with our theoretical interpretation, in
which probability mass is progressively reallocated toward correct
strategies, and the capability gap between each generator and the
strongest verifier narrows over training.}

\section{Ablation Study}

To isolate the effect of \method\ from cross-model knowledge transfer and to
ensure that the observed gains are not simply the consequence of a stronger
model guiding a weaker one, we conduct an ablation study using two models
from the same family with different scales:
\textbf{LLaMA-3.2-1B-Instruct} and \textbf{LLaMA-3.2-3B-Instruct}.
Both models are jointly fine-tuned on MATH-500-Numeric following the
\method\ procedure described in \Cref{sec.model}, including PMI-based
modulation, model-order shuffling across epochs, and the same optimization
setup.

\begin{figure*}[t]
\centering

\begin{subfigure}[t]{0.4\linewidth}
    \centering
    \includegraphics[width=\linewidth]{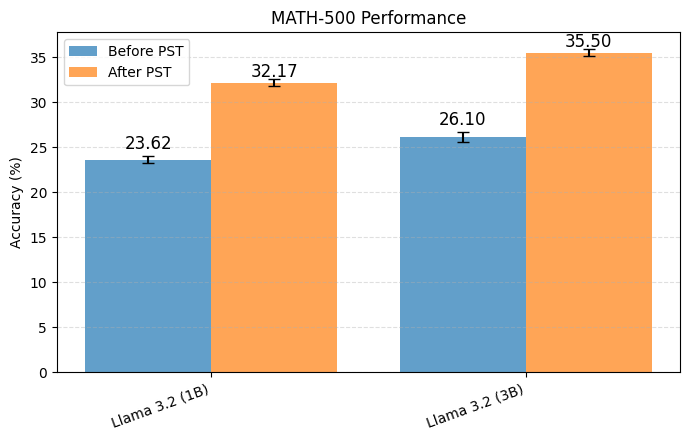}
    \caption{
    Ablation study on MATH-500-Numeric using LLaMA 3.2 models.
    }
    \label{fig:ablation-math500}
\end{subfigure}
\hfill
\begin{subfigure}[t]{0.4\linewidth}
    \centering
    \includegraphics[width=\linewidth]{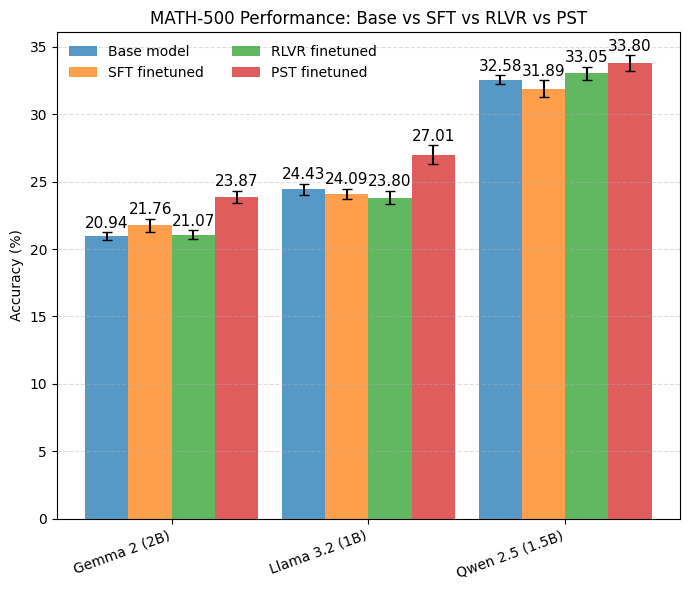}
    \caption{
    Comparison on MATH-500 between SFT, GRPO, and \method.
    }
    \label{fig:baseline-math500}
\end{subfigure}
\hfill
\begin{subfigure}[t]{0.4\linewidth}
    \centering
    \includegraphics[width=\linewidth]{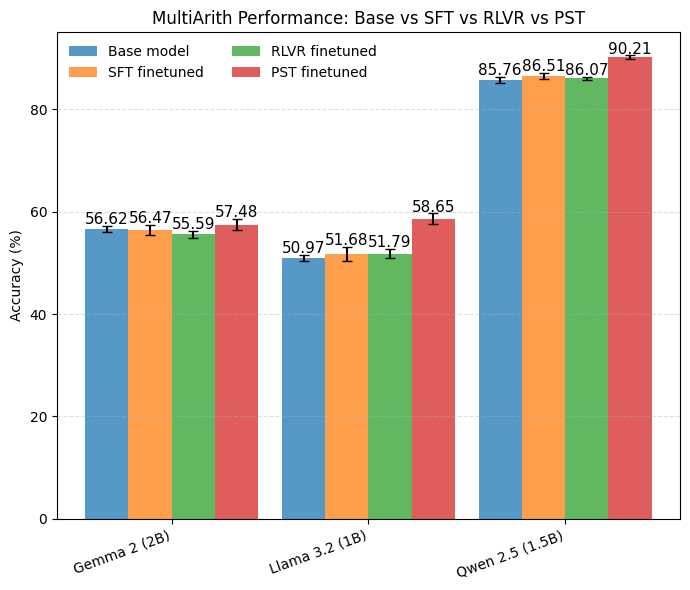}
    \caption{
    Comparison on MultiArith between SFT, GRPO, and \method.
    }
    \label{fig:baseline-multiarith}
\end{subfigure}
\hfill
\begin{subfigure}[t]{0.4\linewidth}
    \centering
    \includegraphics[width=\linewidth]{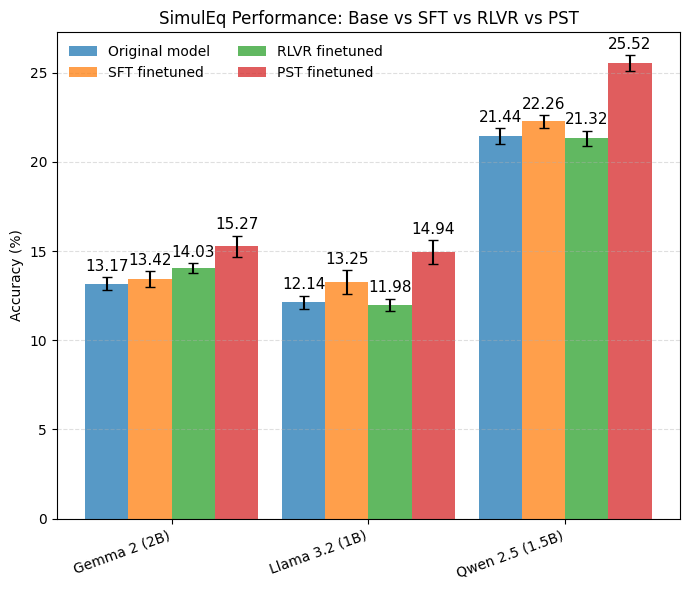}
    \caption{
    Comparison on SimulEq between SFT, GRPO, and \method.
    }
    \label{fig:baseline-simuleq}
\end{subfigure}

\caption{
Results across model scales and benchmarks.
}
\label{fig:all-four}
\end{figure*}

Figure~\ref{fig:ablation-math500} reports the pre-\method\ and
post-\method\ accuracies for both models.
We observe consistent improvements across the two scales.
The 1B model shows a substantial increase in accuracy, and the stronger
3B model also improves despite its much higher initial performance.
This observation suggests that \method\ does not rely on a teacher-like
relationship between weaker and stronger models.
Instead, both models benefit from the PMI-guided multi-model training loop,
which appears to amplify helpful reasoning trajectories already present in
each model and to reduce the impact of uninformative steps.
The fact that both models improve together suggests that \method\ functions
as a cooperative self-refinement process rather than a form of implicit
distillation.

To ensure that the observed gains of \method\ are not simply due to additional
fine-tuning steps or exposure to external supervision, we further compare
against strong single-model fine-tuning baselines constructed under a matched
optimization budget.
Specifically, we sample 500 training instances from
OpenMathInstruct-2~\cite{openmathinstruct2} and fine-tune each model
independently using either supervised fine-tuning (SFT)~\cite{ouyang2022training}
or group relative policy optimization (GRPO)~\cite{shao2024deepseekmath}.
All baselines are initialized from the same instruction-tuned checkpoints
as \method\ and are trained using the same optimizer, learning-rate schedule,
number of updates, batch size, and decoding configuration.
Crucially, baseline training does not involve any multi-model interaction,
PMI-based modulation, or cross-model feedback.

After baseline fine-tuning, all models are evaluated on the same three
downstream numerical reasoning benchmarks: MATH-500-Numeric, MultiArith,
and SimulEq.
The SFT and GRPO baselines are trained only on OpenMathInstruct-2 and do not
use labels from the target benchmarks.
In contrast, \method\ uses only unlabeled prompts from the target benchmark
and never accesses ground-truth answers.
Thus, our comparison evaluates whether label-free, transductive
self-improvement through cross-model interaction can outperform standard
single-model fine-tuning with external supervised or reward-based training
data.

Figures~\ref{fig:baseline-math500},~\ref{fig:baseline-multiarith}, and
\ref{fig:baseline-simuleq} compare the performance of these baselines with
models trained using \method.
Across all three benchmarks, both supervised fine-tuning (SFT) and group
relative policy optimization (GRPO) lead to only modest or inconsistent gains
over the original instruction-tuned models, and their effects vary across
model families and datasets.
In contrast, \method\ consistently improves accuracy across all models and
benchmarks, even when compared against the stronger GRPO baseline.
This consistent gap suggests that the gains achieved by \method\ cannot be
explained by improved single-model optimization or better reward shaping alone.
Instead, they arise from the cooperative, PMI-guided multi-model training
dynamics in \method, which leverage cross-model predictiveness to amplify
informative reasoning trajectories while suppressing uninformative ones.

We provide additional ablations in Appendix~\ref{app:more_baselines} and
Appendix~\ref{app:epoch_ablation}.
These include comparisons with self-consistency, direct self-training,
majority-vote ensemble baselines, and majority-vote self-training, as well as
epoch-wise trajectories from Epoch 1 to Epoch 5.
The additional results show that \method\ improves not only single-sample
accuracy but also sampled-solution quality and ensemble performance, while the
epoch-wise study suggests stable but gradually saturating gains over repeated
\method\ updates.

\section{Conclusion and Future Directions}

We propose Peer-Predictive Self-Training (PST), a fully unsupervised
framework for continued self-improvement of language models based on
cross-model interaction.
By aggregating sequential generations from heterogeneous models and
using PMI-based modulation to scale self-training updates, PST enables
models to leverage peer predictiveness as an internal supervisory signal.
Across multiple mathematical reasoning benchmarks, PST consistently
improves exact-match accuracy and reduces the generator--verifier gap,
without relying on external supervision, reward models, or fixed
teacher--student hierarchies.

More broadly, our results show that peer-predictive feedback enables
stable self-improvement in reasoning-oriented tasks by providing a
reliable source of internal supervision through aggregation and mutual
predictiveness, even without ground truth.

These results also suggest several directions for future work: scaling PST to
larger models and datasets, evaluating generalization on held-out prompts,
studying deeper multi-round interactions and stability criteria, extending the
approach beyond mathematical reasoning, and exploring alternative
peer-predictive modulation schemes.
\newpage



\bibliography{ref}
\bibliographystyle{colm2026_conference}

\newpage
\appendix
\onecolumn
\section*{Appendices}

\section{Proof of \texorpdfstring{\Cref{thm:pst_convergence}}{Theorem 5.1}}
\label{app.proof_pst_convergence}

\begin{proof}
For any $i$, the right-hand side of~\eqref{eq:pst_dyn} satisfies
$\frac{ds_i}{dt} \ge 0$ because $(1-s_i)\ge 0$ on $[0,1]$ and $[v-s_i]_+ \ge 0$.
Hence $s_i(t)$ is monotone non-decreasing.
Moreover, if $s_i(t)=1$ then $\frac{ds_i}{dt}=0$, so $s_i(t)$ cannot exceed $1$.
If $s_i(t)=0$ then $\frac{ds_i}{dt}\ge 0$, so it cannot become negative.
Thus $s_i(t)\in[0,1]$ for all $t\ge 0$.

Similarly, $\frac{dv}{dt} = \beta(1-v)\max_i s_i \ge 0$ whenever $v\in[0,1]$,
so $v(t)$ is monotone non-decreasing.
Also, if $v(t)=1$ then $\frac{dv}{dt}=0$, so $v(t)$ cannot exceed $1$; if $v(t)=0$
then $\frac{dv}{dt}\ge 0$, so it cannot become negative.
Thus $v(t)\in[0,1]$ for all $t\ge 0$.

Since each $s_i(t)$ is monotone non-decreasing and bounded in $[0,1]$, the limit
$s_i^\ast=\lim_{t\to\infty}s_i(t)$ exists for every $i$.
Likewise, since $v(t)$ is monotone non-decreasing and bounded in $[0,1]$, the
limit $v^\ast=\lim_{t\to\infty}v(t)$ exists.

Fix any model $i$ with $s_i(0) < v^\ast$.
Suppose for contradiction that $s_i^\ast < v^\ast$.
Let $\delta := \frac{1}{3}(v^\ast - s_i^\ast) > 0$.
By convergence, there exists $T$ such that for all $t\ge T$,
\[
v(t) \ge v^\ast - \delta
\quad\text{and}\quad
s_i(t) \le s_i^\ast + \delta.
\]
Therefore, for all $t\ge T$,
\[
v(t) - s_i(t) \ge (v^\ast - \delta) - (s_i^\ast + \delta)
= (v^\ast - s_i^\ast) - 2\delta
= \delta,
\]
so $[v(t)-s_i(t)]_+ \ge \delta$.
Also, since $s_i(t)\le s_i^\ast+\delta < 1$ for $t$ large enough (because
$s_i^\ast\le 1$ and $\delta>0$), we have $1-s_i(t) \ge 1-(s_i^\ast+\delta) =: c > 0$
for all sufficiently large $t$.
Plugging into~\eqref{eq:pst_dyn} yields, for all sufficiently large $t$,
\[
\frac{ds_i}{dt} \;\ge\; \alpha\,c\,\delta \;>\; 0,
\]
which implies that $s_i(t)$ must increase by at least $\alpha c\delta\,(t-T)$ for
$t\ge T$, contradicting the existence of the finite limit $s_i^\ast$.
Hence the assumption $s_i^\ast < v^\ast$ is impossible, and we conclude
$s_i^\ast = v^\ast$.

This proves the third claim and completes the proof.
\end{proof}

\section{Additional Experiments and Details}

\submission{\subsection{Empirical Improvements on GV-Gap}
\label{app:gv}

\begin{figure}[htbp]
    \centering
    \begin{subfigure}[b]{0.32\linewidth}
        \includegraphics[width=\linewidth]{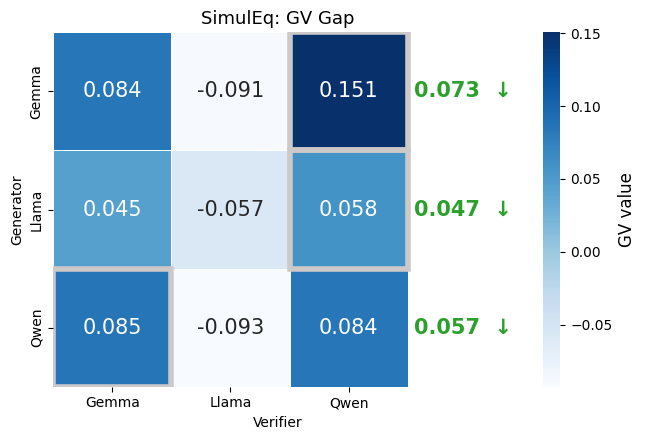}
        \caption{SimulEq}
    \end{subfigure}
    \hfill
    \begin{subfigure}[b]{0.32\linewidth}
        \includegraphics[width=\linewidth]{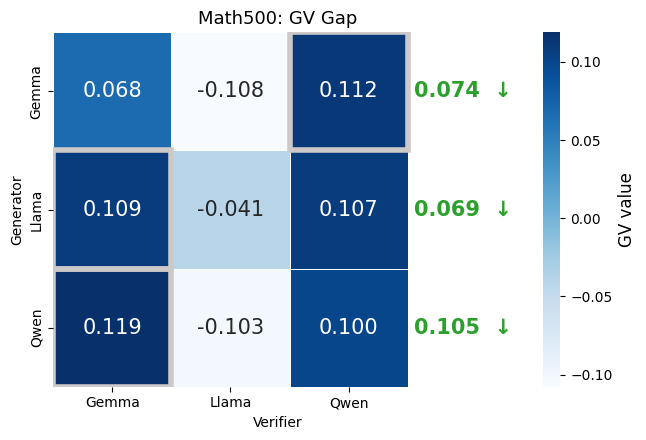}
        \caption{MATH-500-Numeric}
    \end{subfigure}
    \hfill
    \begin{subfigure}[b]{0.32\linewidth}
        \includegraphics[width=\linewidth]{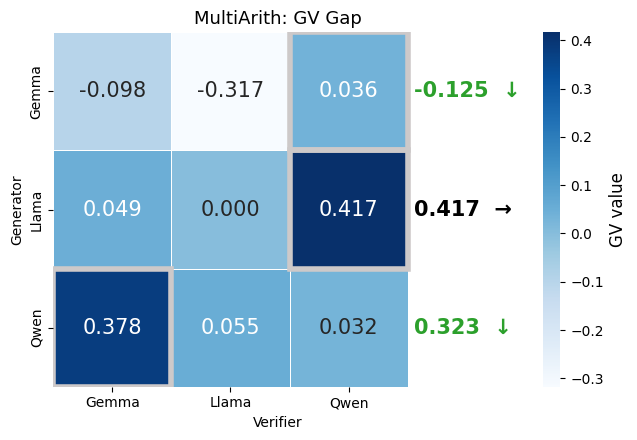}
        \caption{MultiArith}
    \end{subfigure}

    \caption{
    Initial GV matrices on SimulEq, MATH-500-Numeric, and MultiArith, computed using
    initial models as both generators and verifiers.
    Row-wise maxima at initialization are highlighted by gray boxes.
    The \textcolor{ForestGreen}{green} numbers on the right report the
    corresponding row-wise maximum GV gaps after fine-tuning, computed
    between fine-tuned generators and fine-tuned verifiers
    (\textcolor{ForestGreen}{green} indicates a decrease, while black indicates no change).
    }
    \label{fig:gv_initial}
\end{figure}

The analysis above predicts that the gains of \method\ arise from
cross-model interactions, in which stronger verifiers provide a
cleaner training signal.

Empirically, for a generator $M_g$ and a verifier $M_v$, given a
question--answer pair $(x,a)$, we sample $k$ candidate responses
$\{y^{(1)},\dots,y^{(k)}\}\sim p_{M_g}(\cdot\mid x)$ from the generator.
Let $y^{(1)}$ denote the generator's ordinary single-sample output. The
verifier selects the highest-scoring candidate from the same candidate set,
and we define the generator--verifier gap on $x$ as
\begin{align*}
y_v^\star
&= \arg\max_{t\in[k]} s_{M_v}(x,y^{(t)}),\\
\mathrm{GV\text{-}Gap}(M_g,M_v\mid x)
&\triangleq
\mathbf{1}\{\mathrm{Ans}(y_v^\star)=a\}
-
\mathbf{1}\{\mathrm{Ans}(y^{(1)})=a\}.
\end{align*}
Thus, the GV gap measures the additional correctness obtained by
verifier-based selection over the generator's single-sample output.

Here $s_{M_v}(x,y)\in\mathbb{R}$ denotes the scalar score assigned by the
verifier model $M_v$ to a completed response $y$ for input $x$. The verifier
is run in evaluation mode, performing deterministic forward passes on fixed
input--response pairs without sampling, dropout, or parameter updates. The
score is computed as a sequence-level value by summing token-level
log-probabilities over the response tokens and is used only to rank candidate
responses, following the generator--verifier framework of
\cite{song2024mind}. We report the dataset-level GV gap by averaging
$\mathrm{GV\text{-}Gap}(M_g,M_v\mid x)$ over prompts $x$.

Figure~\ref{fig:gv_initial} shows the initial GV matrices on SimulEq,
MATH-500-Numeric, and MultiArith.
Rows correspond to generators and columns to verifiers.
For each generator, the row-wise maximum GV gap at initialization is
highlighted by a gray box.
In most cases, this maximum occurs off the diagonal and is strictly
positive, indicating that a generator aligns more strongly with another
model acting as a verifier than with itself.
This reveals systematic cross-model asymmetries that \method\ can
exploit during training, suggesting that peer interaction provides
stronger supervision than self-training within a single model. 
This aligns with prior findings in \citep{song2024mind, lu2025verification} that cross-verification from diverse models is effective.

The number shown to the right of each row reports the corresponding
row-wise maximum GV gap after fine-tuning, where fine-tuned models are
used as both generators and verifiers.
Across all datasets, these values consistently decrease after training.
Specifically, the sum of row-wise maximum GV gaps is reduced by approximately 39.8\% on SimulEq, 27.1\% on MATH-500-Numeric, and 26.0\% on MultiArith, indicating that generators internalize strategies aligned with the strongest verifiers.
In contrast, standard self-training relies on each model's own predictions and evaluations, whose generator--verifier value is inherently bounded by its own capability and thus cannot match the maximum GV gap achieved by stronger peer models.
By leveraging cross-model interactions, PST allows each generator to align with signals induced by stronger verifiers, leading to systematically reduced generator--verifier gaps.
This behavior is consistent with our theoretical interpretation, in
which probability mass is progressively reallocated toward correct
strategies, and the capability gap between each generator and the
strongest verifier narrows over training. }

\subsection{Experiments on Larger-Scale LLMs}
\label{app:larger_scale}

To further evaluate the scalability of PST, we conduct additional experiments on larger-scale language models.
Except for the use of larger models, all experimental settings remain identical to those described in \Cref{sec.exp} and \Cref{app:repro_details}, including datasets, training procedure, optimization, and evaluation protocols.

Specifically, we follow the same sequential generation process, PMI-based modulation, and LoRA-based fine-tuning setup as in the main experiments.

\begin{figure*}[t]
  \centering
    \begin{subfigure}[t]{0.32\linewidth}
    \centering
    \includegraphics[width=\linewidth]{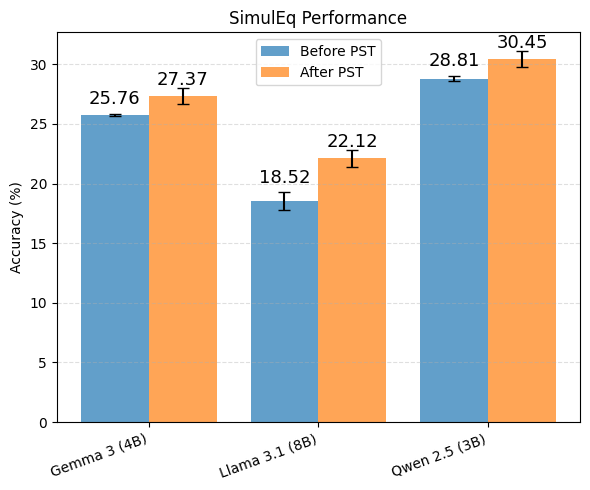}
    \caption{SimulEq}
  \end{subfigure}
   \hfill
  \begin{subfigure}[t]{0.32\linewidth}
    \centering
    \includegraphics[width=\linewidth]{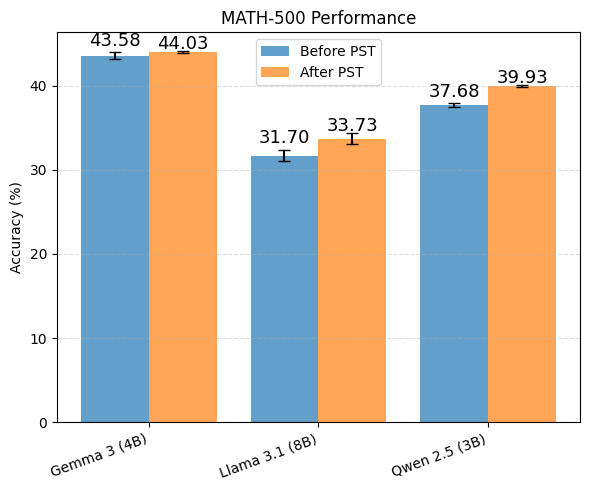}
    \caption{MATH-500-Numeric}
  \end{subfigure}
  \hfill
  \begin{subfigure}[t]{0.32\linewidth}
    \centering
    \includegraphics[width=\linewidth]{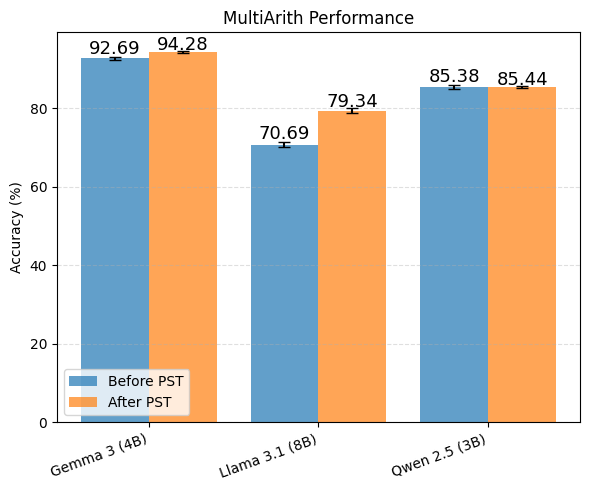}
    \caption{MultiArith}
  \end{subfigure}

  \caption{
  Accuracy before and after fine-tuning on SimulEq, MATH-500-Numeric, and MultiArith.
  Each panel reports pre- and post-fine-tuning performance under \method\ for
  Gemma-3-4B,  LLaMA-3.1-8B, and Qwen2.5-3B.
  }
  \label{fig:accuracy_all_large}
\end{figure*}

We observe consistent performance improvements of PST over the original models in \Cref{fig:accuracy_all_large}, consistent with the trends reported in \Cref{sec.exp}.

\subsection{More Baseline Comparisons}
\label{app:more_baselines}

We provide additional baseline comparisons to distinguish PST from standard
self-consistency and self-training methods. These comparisons evaluate whether
the gains of PST can be explained by sampling more responses, by ordinary
self-training, or by majority-vote pseudo-labeling.

\begin{table*}[t]
\centering
\small
\begin{tabular}{llccc}
\toprule
Dataset & Model & 4 Samples & 8 Samples & 16 Samples \\
\midrule

\multirow{3}{*}{MATH-500}
& Gemma-2-2B & 38.05 $\rightarrow$ 41.51 & 44.03 $\rightarrow$ 49.69 & 52.52 $\rightarrow$ 56.60 \\
& LLaMA-3.2-1B & 42.14 $\rightarrow$ 44.97 & 51.57 $\rightarrow$ 54.72 & 61.95 $\rightarrow$ 64.47 \\
& Qwen2.5-1.5B & 53.46 $\rightarrow$ 54.40 & 60.38 $\rightarrow$ 63.21 & 70.44 $\rightarrow$ 73.27 \\

\midrule

\multirow{3}{*}{SimulEq}
& Gemma-2-2B & 27.16 $\rightarrow$ 35.80 & 39.51 $\rightarrow$ 48.56 & 51.44 $\rightarrow$ 58.85 \\
& LLaMA-3.2-1B & 28.40 $\rightarrow$ 31.69 & 36.63 $\rightarrow$ 42.39 & 45.68 $\rightarrow$ 53.09 \\
& Qwen2.5-1.5B & 37.45 $\rightarrow$ 42.80 & 43.21 $\rightarrow$ 51.44 & 55.56 $\rightarrow$ 59.26 \\

\midrule

\multirow{3}{*}{MultiArith}
& Gemma-2-2B & 86.90 $\rightarrow$ 88.62 & 93.45 $\rightarrow$ 96.21 & 98.97 $\rightarrow$ 98.97 \\
& LLaMA-3.2-1B & 84.48 $\rightarrow$ 89.66 & 94.48 $\rightarrow$ 95.52 & 98.28 $\rightarrow$ 98.97 \\
& Qwen2.5-1.5B & 98.28 $\rightarrow$ 99.31 & 99.66 $\rightarrow$ 99.66 & 99.66 $\rightarrow$ 100.00 \\

\bottomrule
\end{tabular}
\caption{
Self-consistency accuracy (\%) before and after PST training. Each entry is
reported as Original $\rightarrow$ PST.
}
\label{tab:pst_self_consistency}
\end{table*}

First, we evaluate self-consistency performance with 4, 8, and 16 sampled
responses. As shown in \cref{tab:pst_self_consistency}, PST improves
self-consistency accuracy in most settings. This indicates that PST improves
not only single-sample accuracy, but also the quality of the model's sampled
solution set.

\begin{table*}[t]
\centering
\small
\begin{tabular}{llcc}
\toprule
Dataset & Model & Self-Training Baseline & PST \\
\midrule

\multirow{3}{*}{MATH-500}
& Gemma 2 (2B) 
& $23.18 \pm 1.61$ 
& $\mathbf{23.87 \pm 1.37}$ \\

& LLaMA 3.2 (1B)
& $24.65 \pm 1.38$ 
& $\mathbf{26.73 \pm 1.95}$ \\

& Qwen 2.5 (1.5B) 
& $31.98 \pm 2.10$ 
& $\mathbf{33.80 \pm 1.69}$ \\

\midrule

\multirow{3}{*}{SimulEq}
& Gemma 2 (2B) 
& $13.05 \pm 1.55$ 
& $\mathbf{15.27 \pm 1.83}$ \\

& LLaMA 3.2 (1B)
& $12.72 \pm 1.52$ 
& $\mathbf{14.94 \pm 2.03}$ \\

& Qwen 2.5 (1.5B) 
& $20.21 \pm 1.58$ 
& $\mathbf{25.52 \pm 1.43}$ \\

\midrule

\multirow{3}{*}{MultiArith}
& Gemma 2 (2B) 
& $57.28 \pm 2.90$ 
& $\mathbf{57.48 \pm 3.21}$ \\

& LLaMA 3.2 (1B)
& $53.59 \pm 2.62$ 
& $\mathbf{58.66 \pm 3.08}$ \\

& Qwen 2.5 (1.5B) 
& $83.21 \pm 1.59$ 
& $\mathbf{90.21 \pm 1.07}$ \\

\bottomrule
\end{tabular}
\caption{
Accuracy (\%) comparison between the self-training baseline and PST. Results
are reported as mean $\pm$ standard deviation across random seeds.
}
\label{tab:pst_vs_selftraining_all}
\end{table*}

Second, we compare PST with a direct self-training baseline. This baseline
fine-tunes the model on pseudo-labels generated from its own outputs, using the
same hyperparameters as PST. As shown in \cref{tab:pst_vs_selftraining_all},
PST consistently outperforms this baseline across datasets and model families.
Thus, the improvement of PST is not simply due to additional fine-tuning on
model-generated answers.

\begin{table}[t]
\centering
\small
\begin{tabular}{lccc}
\toprule
Dataset & Before PST MV & After PST MV & Best Post-PST Single \\
\midrule
MATH-500 & $25.79$ & $30.92$ & $\mathbf{33.80}$ \\
SimulEq & $15.09$ & $16.05$ & $\mathbf{25.52}$ \\
MultiArith & $72.87$ & $79.77$ & $\mathbf{90.21}$ \\
\bottomrule
\end{tabular}
\caption{
Majority-vote performance before and after PST. MV denotes majority voting
among the three models. The best post-PST single model is Qwen 2.5 (1.5B).
}
\label{tab:majority_vote_before_after_pst}
\end{table}

Third, we examine whether PST also improves model ensembles. In
\cref{tab:majority_vote_before_after_pst}, majority voting among the three
models improves after PST on all datasets. Moreover, the best single model
after PST outperforms the original pre-PST majority-vote ensemble. This
suggests that PST can produce a single model that is stronger than the original
model ensemble.

\begin{table*}[t]
\centering
\small
\begin{tabular}{llccc}
\toprule
Dataset & Model & Before PST & PST & Majority Vote + SFT \\
\midrule

\multirow{3}{*}{SimulEq}
& Gemma 2 (2B) & $13.17$ & $\mathbf{15.27}$ & $13.99$ \\
& LLaMA 3.2 (1B) & $12.14$ & $\mathbf{14.94}$ & $12.76$ \\
& Qwen 2.5 (1.5B) & $21.44$ & $\mathbf{25.52}$ & $20.16$ \\

\midrule

\multirow{3}{*}{MATH-500-Numeric}
& Gemma 2 (2B) & $20.94$ & $\mathbf{23.87}$ & $21.38$ \\
& LLaMA 3.2 (1B) & $24.43$ & $\mathbf{26.73}$ & $24.84$ \\
& Qwen 2.5 (1.5B) & $32.58$ & $\mathbf{33.80}$ & $32.39$ \\

\midrule

\multirow{3}{*}{MultiArith}
& Gemma 2 (2B) & $56.62$ & $\mathbf{57.48}$ & $56.55$ \\
& LLaMA 3.2 (1B) & $50.97$ & $\mathbf{58.65}$ & $53.45$ \\
& Qwen 2.5 (1.5B) & $85.76$ & $\mathbf{90.21}$ & $86.90$ \\

\bottomrule
\end{tabular}
\caption{
Comparison between PST and a majority-vote self-training baseline. Majority
Vote + SFT first aggregates sampled responses by majority voting over final
numerical answers, and then fine-tunes on one randomly selected response with
the majority-vote answer.
}
\label{tab:majority_vote_sft_baseline}
\end{table*}

Finally, we compare PST with a majority-vote self-training baseline. This
baseline uses majority voting only to construct pseudo-labels, and then applies
standard SFT. As shown in \cref{tab:majority_vote_sft_baseline}, PST
outperforms Majority Vote + SFT across all datasets and models. This suggests
that PST's gains are not merely due to pseudo-labeling by self-consistency.
Instead, the peer-predictive weighting signal provides additional information
beyond the final majority-vote answer.

\subsection{Ablation Study on Epochs}
\label{app:epoch_ablation}

We further study how PST performance changes across training epochs. The goal
of this ablation is to check whether the improvement of PST is stable across
training, rather than coming from a single lucky checkpoint. We report the
accuracy of each model from Epoch 1 to Epoch 5, where Epoch 5 is used as the
final PST checkpoint in our main experiments.

\begin{table}[t]
\centering
\small
\begin{tabular}{lccccc}
\toprule
Model & Epoch 1 & Epoch 2 & Epoch 3 & Epoch 4 & Epoch 5 / PST \\
\midrule
LLaMA-3.2-1B
& $13.00 \pm 1.94$
& $13.25 \pm 0.94$
& $14.16 \pm 1.50$
& $14.16 \pm 1.94$
& $\mathbf{14.94 \pm 2.03}$ \\

Qwen2.5-1.5B
& $21.40 \pm 1.54$
& $22.14 \pm 2.17$
& $22.80 \pm 2.15$
& $24.61 \pm 2.05$
& $\mathbf{25.52 \pm 1.43}$ \\

Gemma-2-2B
& $13.74 \pm 1.56$
& $\mathbf{15.39 \pm 1.61}$
& $14.32 \pm 1.87$
& $\mathbf{15.39 \pm 2.30}$
& $15.27 \pm 1.83$ \\
\bottomrule
\end{tabular}
\caption{SimulEq accuracy (\%) across PST training epochs. Results are reported as mean $\pm$ standard deviation across random seeds.}
\label{tab:simuleq_pst_epochs}
\end{table}

\begin{table}[t]
\centering
\small
\begin{tabular}{lccccc}
\toprule
Model & Epoch 1 & Epoch 2 & Epoch 3 & Epoch 4 & Epoch 5 / PST \\
\midrule
LLaMA-3.2-1B
& $24.59 \pm 1.53$
& $25.79 \pm 1.82$
& $25.85 \pm 1.03$
& $25.97 \pm 2.63$
& $\mathbf{26.73 \pm 1.95}$ \\

Qwen2.5-1.5B
& $33.02 \pm 1.37$
& $33.02 \pm 2.38$
& $\mathbf{35.22 \pm 1.47}$
& $34.59 \pm 1.82$
& $33.80 \pm 1.69$ \\

Gemma-2-2B
& $21.76 \pm 1.00$
& $23.21 \pm 1.76$
& $\mathbf{24.34 \pm 0.72}$
& $24.09 \pm 0.72$
& $23.87 \pm 1.37$ \\
\bottomrule
\end{tabular}
\caption{MATH-500 accuracy (\%) across PST training epochs. Epoch 5 corresponds to the final PST checkpoint. Results are reported as mean $\pm$ standard deviation across random seeds.}
\label{tab:math500_pst_epochs}
\end{table}

\begin{table}[t]
\centering
\small
\begin{tabular}{lccccc}
\toprule
Model & Epoch 1 & Epoch 2 & Epoch 3 & Epoch 4 & Epoch 5 / PST \\
\midrule
LLaMA-3.2-1B
& $51.17 \pm 1.35$
& $52.62 \pm 2.78$
& $56.00 \pm 1.41$
& $\mathbf{60.07 \pm 1.53}$
& $58.66 \pm 3.08$ \\

Qwen2.5-1.5B
& $87.72 \pm 1.37$
& $87.72 \pm 1.53$
& $88.76 \pm 1.28$
& $89.10 \pm 1.85$
& $\mathbf{90.21 \pm 1.07}$ \\

Gemma-2-2B
& $55.17 \pm 3.02$
& $53.10 \pm 1.22$
& $52.48 \pm 1.64$
& $\mathbf{58.48 \pm 3.50}$
& $57.48 \pm 3.21$ \\
\bottomrule
\end{tabular}
\caption{MultiArith accuracy (\%) across PST training epochs. Epoch 5 corresponds to the final PST checkpoint. Results are reported as mean $\pm$ standard deviation across random seeds.}
\label{tab:multiarith_pst_epochs}
\end{table}

As shown in \cref{tab:simuleq_pst_epochs,tab:math500_pst_epochs,tab:multiarith_pst_epochs},
PST generally improves as training proceeds, especially for LLaMA-3.2-1B and
Qwen2.5-1.5B. On SimulEq, the final checkpoint gives the best result for both
LLaMA and Qwen, while Gemma reaches a similar peak at earlier epochs. On
MATH-500 and MultiArith, some models reach their best accuracy before Epoch 5,
but the final PST checkpoint remains close to the best epoch.

These results also suggest that iterative application of PST gradually
saturates. Performance usually improves in the first few epochs, but later
epochs often show smaller gains or mild fluctuations rather than continuous
large improvements. This behavior is consistent with our theoretical
intuition: the potential improvement of PST is bounded by the available
generator-verifier (GV) gap. Once repeated PST updates reduce this gap on the
same data, further applications are expected to provide only limited additional
benefit unless new data or stronger peer signals are introduced. Overall, the
epoch ablation shows that PST is relatively stable across training and that
its gains are not caused by a single isolated checkpoint, although early
stopping could further improve some model-dataset pairs.

\section{Why Is Verification Easier Than Generation:
A Transformer Architecture Perspective}
\label{sec.hardness}

In this section, we provide theoretical intuition for a fundamental
asymmetry underlying our method: in reasoning-oriented tasks, verifying a
candidate solution is often substantially easier than generating a correct
solution from scratch.
This verification--generation gap enables reliable learning signals even
in the absence of external supervision and plays a central role in our
analysis.

Following \cite{wang2024alpine,wang2025benefits}, we view reasoning as a
\emph{long-range logical deduction problem}.
Such tasks require deriving a final conclusion through a sequence of
logically dependent intermediate steps.
Maintaining global coherence across many steps is inherently challenging
for autoregressive language models and makes direct generation of correct
solutions difficult.

A long-range logical deduction naturally induces a dependency structure.
Each intermediate conclusion corresponds to a reasoning state, and each
valid inference corresponds to a transition between states.
Although this structure may not be explicitly represented, it can be
abstracted as a directed dependency graph whose nodes represent
intermediate reasoning states and whose edges represent valid logical
implications.
Solving the deduction task corresponds to identifying a valid chain of
inferences connecting the initial premises to the final conclusion.

Crucially, there is a qualitative difference between constructing such a
chain and verifying one.
\emph{Generation} requires discovering a globally consistent sequence of
deductions, whereas \emph{verification} only requires checking whether a
proposed sequence is logically valid.
This asymmetry suggests that even when a model struggles to generate a
correct answer, it may still be capable of reliably recognizing correctness
when presented with a candidate solution.

Formally, we represent a long-range logical deduction task by a directed
graph $G = (V, E)$ with a designated source node $s$ and target node $t$.
Nodes correspond to intermediate reasoning states, and directed edges
represent valid inference steps.
Under this abstraction, \emph{generation} corresponds to constructing a
valid sequence of deductions that connects $s$ to $t$, while
\emph{verification} corresponds to checking whether a given candidate
sequence indeed constitutes such a valid connection.
As we show below, these two tasks differ sharply in computational
complexity.

\subsection{Complexity Analysis of the Verification Task}

We first show that the verification task can be carried out by a
constant-depth transformer with polynomial embedding size in the input
length.
Intuitively, verification only requires local consistency checks across a
proposed deduction chain and can be performed in a highly parallel manner.

\begin{lemma}\label{lemma.tc0}
Every problem in the uniform-TC$^0$ family can be simulated by a
constant-depth, log-precision transformer with polynomial embedding size
in the input length.
Conversely, every such transformer can be simulated by a threshold circuit
of polynomial size and constant depth \cite{merrill2023parallelism}.
\footnote{It is conjectured that an $O(\log^k n)$-depth, log-precision
transformer with polynomial embedding size has computational power
equivalent to the class TC$^k$.}
\end{lemma}

\begin{theorem}[Verification lies in TC$^0$]\label{thm.verify}
For both in-context and pre-learned graph settings, the verification task,
namely checking whether a proposed deduction chain is valid and connects the
source to the target, belongs to the complexity class TC$^0$.
Consequently, it can be simulated by a constant-depth, log-precision
transformer with polynomial embedding size.
\end{theorem}

\begin{proof}
We show that the verification task lies in TC$^0$ for both the in-context
and pre-learned graph settings. TC$^0$ consists of problems solvable by
uniform, constant-depth, polynomial-size threshold circuits with AND, OR,
and MAJORITY gates.

Given a graph $G = (V, E)$ and a candidate path $(v_0, v_1, \dots, v_k)$
with $v_0=s$ and $v_k=t$,
verification reduces to checking in parallel whether each pair
$(v_i, v_{i+1})$ is an edge in $E$. Each edge-membership test is a simple
table lookup or adjacency-matrix access, which can be performed in constant
depth using threshold circuits. The final correctness check is the AND of
all edge-validity results, which also remains in TC$^0$.

In the pre-learned setting, the graph is assumed to be stored in the
model's internal representation from training; verifying a candidate path
amounts to matching queried edges against this stored structure. In the
in-context setting, the graph is explicitly provided as part of the input.
In both cases, verification again reduces to parallel edge-existence
queries, each implementable with constant-depth circuitry.

All required operations therefore fit within uniform threshold circuits of
constant depth and polynomial size. Hence the verification task belongs to
TC$^0$, establishing the theorem.
\end{proof}

This result follows because each local inference check
$(v_j, v_{j+1}) \in E$ can be performed in constant parallel time, and the
results can be aggregated using AND, OR, or MAJORITY gates.
In pre-learned graphs, this corresponds to table lookup.
In in-context graphs, it corresponds to explicit consistency checking
within the prompt.
The entire verification procedure therefore fits within TC$^0$.

Thus, determining whether a proposed response is \emph{correct} ($C$) or
\emph{incorrect} ($Z$) is computationally lightweight and well within the
capabilities of low-depth transformers.

\subsection{Complexity Analysis of the Generation Task}

We next analyze the generation task, which requires constructing a valid
deduction chain rather than merely verifying one.
We show that this process is substantially harder and exceeds the
computational power of constant-depth transformers.

\begin{theorem}[Generation exceeds TC$^0$]\label{thm:path}
For both in-context and pre-learned graph settings, constructing a valid
deduction chain is computationally harder than verification:
\begin{itemize}
    \item In and-or graphs, the PATH problem is P-hard. Hence, generation
    in the in-context setting exceeds TC$^0$ unless
    $\text{TC}^0 = \text{P}$.
    \item In or-graphs, the PATH problem is NL-complete. Hence, generation
    exceeds TC$^1$ unless $\text{NL} = \text{TC}^1$.
    \item Even in the pre-learned setting, evaluating deduction chains on
    and-or graphs can simulate monotone circuit evaluation, which is at
    least NL-hard \cite{rossman2014formulas}.
\end{itemize}
\end{theorem}

\begin{proof}
We analyze the computational complexity of constructing a valid
$s$--$t$ path in the three scenarios stated in the theorem.

\paragraph{1. In-context and-or graphs.}
We first show that path construction on and-or graphs in the in-context
setting is P-hard. The reduction is from Horn clause satisfiability, a
P-complete problem. Given a Horn formula $\varphi$, we construct an
and-or graph $G_\varphi$ in polynomial time as follows. Each positive
literal is represented as a seed node, and each negative literal is
represented as a target node. A clause containing one positive literal and
one or more negative literals introduces directed edges from the positive
literal node to the corresponding negative literal nodes. Clauses with
multiple negative literals introduce a new target node that aggregates their
dependencies. Finally, we add a single global target node connected to all
clause-level target nodes. A truth assignment satisfying $\varphi$
corresponds exactly to a path from some seed node to the global target node,
and conversely the existence of such a path yields a satisfying assignment.
The size of $G_\varphi$ is polynomial in $|\varphi|$, so the reduction is
valid. Since Horn satisfiability is P-complete, path construction on
and-or graphs is P-hard. Under the hypothesis $\mathrm{TC}^0 \subsetneq
\mathrm{P}$, this implies that constructing a valid path on and-or graphs
cannot be carried out by constant-depth threshold circuits and therefore
cannot be simulated by constant-depth, log-precision transformers with
polynomial embedding size.

\paragraph{2. In-context or-graphs.}
For or-graphs, the in-context problem is the classical directed reachability
problem: given the explicitly provided graph $G=(V,E)$ and distinguished
nodes $s,t$, determine whether there exists a directed path from $s$ to
$t$. This is the PATH problem, which is NL-complete. Since
\[
\mathrm{NC}^1 \subseteq \mathrm{L} \subseteq \mathrm{NL}
  \subseteq \mathrm{NC}^2
  \qquad\text{and}\qquad
\mathrm{NC}^i \subseteq \mathrm{AC}^i \subseteq \mathrm{TC}^i
  \subseteq \mathrm{NC}^{i+1},
\]
solving an NL-hard problem within $\mathrm{TC}^1$ would imply
$\mathrm{NL} = \mathrm{TC}^1$, an equality that is widely conjectured to be
false. Therefore, assuming $\mathrm{NL} \neq \mathrm{TC}^1$, path construction
on or-graphs in the in-context setting cannot be solved by $\mathrm{TC}^1$
circuits, and a fortiori cannot be simulated by constant-depth,
log-precision transformers.

\paragraph{3. Pre-learned and-or graphs.}
In the pre-learned setting, the transformer is assumed to have memorized the
entire graph structure during training. We show that even under this
assumption, constructing a valid path may require computational power at
least as strong as NL. The key observation is that reachability in an
and-or graph can simulate the evaluation of an arbitrary monotone circuit
of polynomial size. Each gate of the monotone circuit is encoded as a node
in the and-or graph: an AND gate corresponds to an ``and-node'' that
requires all predecessors to be reachable, and an OR gate corresponds to an
``or-node'' requiring at least one predecessor to be reachable. The output
gate of the monotone circuit is designated as the target node whose
reachability determines acceptance. This construction is standard and
preserves polynomial size; see \cite{rossman2014formulas}. Since monotone
circuits of polynomial size can express all NL computations, the existence
of a constant-depth (TC$^0$) implementation for reachability on these
graphs would imply $\mathrm{TC}^0 = \mathrm{NL}$, contrary to prevailing
complexity assumptions. Therefore, unless $\mathrm{TC}^0 = \mathrm{NL}$,
there exist pre-learned and-or graphs for which path construction exceeds
the computational power of constant-depth, log-precision transformers.

\paragraph{Conclusion.}
Across all three settings, constructing a valid path (planning) is at least
as hard as solving P-hard or NL-hard problems. Under standard complexity
separation assumptions ($\mathrm{TC}^0 \subsetneq \mathrm{P}$ and
$\mathrm{NL} \neq \mathrm{TC}^1$), these tasks exceed the computational
power of constant-depth, log-precision transformers with polynomial
embedding size, establishing the theorem.
\end{proof}

These results follow from standard complexity-theoretic reductions.
Horn clause satisfiability, which is P-complete, reduces to PATH on and-or
graphs, establishing P-hardness in the in-context setting.
The classical PATH problem on or-graphs is NL-complete.
In the pre-learned case, deduction-chain evaluation corresponds to
evaluating monotone circuits composed of AND and OR gates, which is known
to be NL-hard.
If constant-depth transformers could solve these tasks, it would imply
$\text{TC}^0 = \text{NL}$ or even $\text{TC}^0 = \text{P}$, which is widely
believed to be unlikely.

\subsection{Discussion and Implications}

The analysis above highlights a fundamental asymmetry between
\emph{verification} and \emph{generation} in long-range logical deduction
tasks.
Verification lies in TC$^0$ and is therefore feasible for constant-depth,
log-precision transformers, whereas generation is P- or NL-hard and
substantially more demanding.

This asymmetry directly motivates the assumption used in
\Cref{subsec:theory} that the aggregated evaluation signal can be initially stronger
than any individual generator.
Since verification is computationally easier than generation, models may
possess relatively strong verification capabilities even when their
generation capabilities are weak.
Aggregating multiple models allows these verification capabilities to be
combined, yielding an effective evaluation signal that satisfies
\[
v(0) > \max_i s_i(0).
\]

As training progresses, improvements in generation reduce this gap,
leading to saturation when models internalize the evaluation signal.

\subsection{Additional Reproducibility Details}
\label{app:repro_details}

This section reports only implementation-level details not specified in
the main text; all training and evaluation code
is publicly available at
\url{https://github.com/chen-lab-seas/Peer-Predictive-Self-Training}.
For evaluation, all models use stochastic decoding with temperature $0.7$
and sampling enabled, and results are averaged over $10$ independent runs
with different random seeds to estimate variance. Error bars denote
$\pm 1$ standard deviation across runs.

For baseline supervised fine-tuning (SFT), each model is trained
independently on $500$ examples from OpenMathInstruct-2 for $5$ epochs
using causal language modeling, with per-device batch size $1$, gradient
accumulation of $8$ steps, and learning rate $10^{-6}$. The learning rate
for all baselines is lightly tuned and found to have no significant
impact; we therefore use the same learning rate as PST to ensure a
consistent comparison.

For verifiable reinforcement learning, we adopt group relative policy
optimization (GRPO) as implemented in the TRL library~\cite{trl_grpo},
which optimizes the policy directly using scalar rewards assigned to
generated completions. The reward is binary and fully verifiable: a
completion receives reward $1$ if the gold numeric answer appears as a
standalone number in the output text and $0$ otherwise. GRPO training uses
$5$ epochs, per-device batch size $1$, gradient accumulation of $8$ steps,
learning rate $10^{-6}$, and samples $4$ completions per prompt during
training, with prompt and completion lengths capped at $512$ and $256$
tokens respectively.

For PST training, models are shuffled at the start of each epoch using a
fixed seed so that the identity of the final aggregating model rotates
across epochs. The final model produces the aggregate response but is
updated only via standard next-token prediction on its own output, while
non-final models are updated by next-token cross-entropy on their own
generated responses, scaled by a sigmoid function of the PMI-based signal
with temperature $\tau=3.0$. Gradients are clipped to global norm $1.0$,
and optimization steps are skipped whenever non-finite gradients are
detected.

\end{document}